\documentclass[letterpaper, 10 pt, conference]{ieeeconf}  

\IEEEoverridecommandlockouts                              
\usepackage{amssymb}
\usepackage{graphicx}
\usepackage{xcolor}

\usepackage{booktabs}
\usepackage{subcaption}

\newtheorem{definition}{Definition}[section]

\newtheorem{lemma}{Lemma}[section]
\usepackage{amsmath}
\usepackage{algorithm}       
\usepackage{algpseudocode}   

\title{\LARGE \bf
Learning to Adapt: Representation-Based
Reinforcement Learning for Multi-Task
Skill Transfer
}

\author{Aryan Naveen, Haitong Ma, Haldun Balim, Na Li
\thanks{This work was supported by NSF AI institute: 2112085, NSF ECCS-2401390, ONR N000142512173}
\thanks{A. Naveen is currently affiliated with the Massachusetts Institute of Technology. This work was primarily carried out during his undergraduate studies at the Harvard School of Engineering and Applied Sciences. {\tt\small aryannav@mit.edu} }%
\thanks{H. Ma, H. Balim, and N. Li are with Harvard School of Engineering and Applied Sciences, 
        {\tt\small haitongma@g.harvard.edu, hbalim@fas.harvard.edu, nali@seas.harvard.edu}}%
}

\begin{document}

\maketitle
\thispagestyle{empty}
\pagestyle{empty}

\begin{abstract}

   Reinforcement learning has achieved remarkable success in learning complex control policies, yet its applicability remains limited due to sample inefficiency and poor generalization across tasks. In this work, we propose RepMT-SAC, a framework for multi-task RL that enables efficient knowledge sharing and robust transfer to new tasks. RepMT-SAC uses spectral MDP decomposition to capture transferable dynamics, structuring the value function into a task-agnostic core with a minimal task-specific adjustment. This design allows for strong zero-shot performance on in-distribution tasks and rapid few-shot adaptation to out-of-distribution tasks. We evaluate RepMT-SAC on quadcopter trajectory-following tasks across in-distribution and out-of-distribution contexts, demonstrating that it outperforms baselines by up to 30\%. 
\end{abstract}

\section{INTRODUCTION}
Reinforcement learning (RL) has emerged as a powerful framework for learning control policies in complex, high-dimensional systems  \cite{tang2025deep, mnih2015human}. While traditional formulations focus on single-task settings with well-specified reward functions, many potential applications are inherently multi-task, requiring agents to solve a distribution of related Markov decision processes (MDPs). In such settings, tasks often exhibit substantial shared structure—arising from common environment dynamics, state representations, or underlying physical laws—suggesting that knowledge acquired from one task should, in principle, be transferable to others.

Despite this potential, standard RL approaches largely treat tasks independently \cite{DBLP:journals/corr/abs-1801-01290}, leading to policies that overfit to task-specific reward signals and fail to exploit shared structure across tasks. This results in poor generalization, limited zero-shot transfer, and high sample complexity when adapting to new tasks \cite{yu2020gradient, dulac2020empirical}. The core challenge lies in disentangling task-specific objectives from task-invariant structure, such as shared dynamics, to enable the learning of reusable representations.

In this work, we focus on a class of multi-task MDPs in which all tasks share the same underlying system dynamics but differ in their reward functions or reference objectives. This setting is common in control applications—for example, a quadrotor required to track different trajectories operates under identical dynamics across tasks. Such shared dynamics induce a latent structure that can be leveraged to enable efficient knowledge transfer and improved generalization. However, conventional RL methods often entangle this structure with task-specific learning signals, preventing effective reuse of learned representations. Our goal is to explicitly exploit this shared-dynamics structure to improve sample efficiency and generalization across tasks.

Recent work in multi-task RL has begun to explore leveraging shared structure through latent task inference and value decomposition. Methods such as \cite{offpolicymeta, zintgraf2019varibad} model tasks as latent variables, enabling fast adaptation but implicitly entangling dynamics and rewards within learned embeddings. In contrast, other approaches factorize value functions into shared dynamics-dependent features and task-specific rewards \cite{DBLP:journals/corr/BarretoMSS16}, but these features are conditioned on a specific policy and capture only aspects of the dynamics relevant to expected returns. While these extensions enhance scalability and stability, they primarily rely on implicit sharing or value-based decompositions. In contrast, our approach leverages recent advances in spectral decompositions of MDPs to explicitly parameterize task-invariant dynamics as a standalone object, enabling disentanglement without depending on latent embeddings or value functions.

Recent work in theoretical reinforcement learning has explored spectral decompositions of Markov Decision Processes (MDPs) \cite{jin2020provably} \cite{yang2020reinforcement} \cite{uehara2021representation}, offering a framework for learning environment-aware representations. The resulting spectral functions of the transition dynamics in MDPs can \textit{linearly} model the state-action value function. These representations can be leveraged to efficiently synthesize policies for the modeled task, and when the representations are unknown, several sample-efficient learning methods—such as maximum likelihood estimation \cite{uehara2021representation}, contrastive learning \cite{zhang2022making}, or variational approaches—have been proposed to discover them from interaction data \cite{ren2022latent} \cite{ren2022spectral}.

Nevertheless, while existing spectral decomposition methods aim to learn task-independent representations, in practice these representations often conflate task-specific and task-agnostic information \cite{uehara2021representation, ren2022latent}. This entanglement limits their ability to generalize across tasks, as the learned skills may encode unnecessary task-dependent details. In this work, we propose an alternative perspective: the Q-function for any task can be expressed as a linear combination of truly task-agnostic skills, each weighted by the task. By explicitly separating skill representations from task-specific objectives, our approach enables reusable, transferable policies and more efficient adaptation to new tasks.

Building on this insight, we propose \textbf{Rep}resentation \textbf{M}ulti-\textbf{T}ask learning framework (\textbf{RepMT}), a multi-task framework designed to explicitly disentangle task-invariant environment representations from task-specific objectives. In the upstream phase, RepMT learns representations for task-agnostic dynamics that capture reusable skills across the environment. During task-specific finetuning, the framework learns task-dependent weights that combine these skills into an effective policy for each task.

RepMT is tested on quadcopter trajectory-following tasks spanning both in domain and out of domain settings. It achieves strong zero-shot performance on in-distribution tasks and rapidly adapts to out-of-distribution tasks with minimal finetuning, outperforming SAC \cite{DBLP:journals/corr/abs-1801-01290} and CTRL-SAC \cite{zhang2022making}.

The remainder of this paper is organized as follows. Section~\ref{sec: preliminaries} introduces the necessary notation, the multi-task setting. Section~\ref{sec: mt mdp sd} extends existing spectral decomposition to the multi-task setting and subsequently Section \ref{sec:repmt} presents RepMT-SAC, covering the critic's structural parameterization, upstream learning of dynamics and task representations, and downstream finetuning for rapid adaptation. Section~\ref{sec: experiments} evaluates RepMT-SAC on quadcopter trajectory-following tasks, demonstrating substantial improvements in reward, task success rate, and adaptation speed over SAC and prior representation-based RL methods.

\section{Preliminaries}
\label{sec: preliminaries}
In this section, we first formalize the multi-task reinforcement learning problem and introduce the task-conditioned MDP framework. We then present our grounding application of quadrotor trajectory tracking, followed by a brief overview of spectral representations and their relevance to MDP decomposition.
\subsection{Problem Setup}
In this work, we extend the standard MDP formulation to capture a family of stochastic control problems parameterized by a task variable $\tau \in \mathcal{T}$. Formally, we model the multi-task setting as a collection of Markov Decision Processes (MDPs) indexed by $\tau$, where $\mathcal{T}$ denotes a measurable task space. Each task $\tau$ induces an MDP of the form:
\[
\mathcal{M}^{(\tau)} = (\mathcal{S}, \mathcal{A}, r_\tau, P, \rho, \gamma),
\]
where $\mathcal{S}$ and $\mathcal{A}$ denote the shared state and action spaces, respectively; $P$ is the shared transition kernel, $\rho$ is the initial state distribution, and $\gamma \in [0,1)$ is the discount factor. While the state and action spaces, as well as the dynamics $P$ and initial distribution $\rho$, are assumed to be task-invariant, the reward function $r_\tau: \mathcal{S} \times \mathcal{A} \to \mathbb{R}$ is task-dependent, which we denote explicitly as
\[
r_\tau(s,a) := r(s,a,\tau), \quad \forall s \in \mathcal{S}, a \in \mathcal{A}.
\]

\begin{definition}[Task Distribution]
We assume a probability distribution $\mu \in \Delta(\mathcal{T})$ over tasks. Sampling from $\mu$ induces a stochastic process over MDPs: first, a task $\tau \sim \mu$ is drawn, and subsequently, a policy $\pi$ is executed in the corresponding MDP $\mathcal{M}^{(\tau)}$.
\end{definition}

Given the task distribution $\mu$ and the family of MDPs $\{\mathcal{M}^{(\tau)}\}_{\tau \in \mathcal{T}}$, we consider task conditioned policies: $\pi: \mathcal{S} \times \mathcal{T} \to \Delta(\mathcal{A})$, which explicitly condition the policy on the task variable $\tau$.

In this paper, we focus on task-conditioned policies. For a given task-conditioned policy $\pi$, we define the task-specific action-value function as:
\[
V^{\pi}(s;\tau) := \mathbb{E} \Bigg[ \sum_{t=0}^{\infty} \gamma^t \, r(s_t, a_t, \tau) \;\Big|\; s_0 = s \Bigg], \\
\]

\[
Q^{\pi}(s,a;\tau) := \mathbb{E}_{\pi, P} \left[ \sum_{t=0}^{\infty} \gamma^t r(s_t, a_t, \tau) \,\middle|\, s_0 = s, a_0 = a \right].
\]

The goal of multi-task reinforcement learning (MTRL) is to optimize the expected return across tasks drawn from $\mu$. Formally, the MTRL objective is defined as
\[
J(\pi) := \mathbb{E}_{\tau \sim \mu} \Big[ \mathbb{E}_{s_0 \sim \rho} [V^{\pi}(s_0;\tau)] \Big].
\]

\begin{definition}[MTRL Optimal Policies]
We formally denote the MTRL optimal policy as the policy that maximizes the expected return across a task distribution $\mu$:
    \[
    \pi^* \in \arg\max_{\pi} \; J(\pi).
    \]
\end{definition}

By conditioning policies on the task variable, one can learn strategies that adapt to diverse objectives while leveraging shared structure across tasks. To put these optimal policies into practice, we must ensure that the multi-task formulation remains well-behaved while preserving the underlying Bellman properties. The existence and uniqueness of these task-conditioned value functions are grounded in the theoretical properties of the multi-task Bellman operators. 

Specifically, for each task $\tau \in \mathcal{T}$, we define the task-specific Bellman operator $\mathcal{T}_\tau^\pi$ acting on value functions as
\begin{equation}
(\mathcal{T}_\tau^\pi V)(s) := \mathbb{E}_{a \sim \pi(\cdot|s,\tau)} \left[ r(s,a,\tau) + \gamma \mathbb{E}_{s' \sim P(\cdot|s,a)} V(s') \right].
\end{equation}

Similarly, to bridge the gap toward the globally optimal task-conditioned policy $\pi^*$, we define the optimal Bellman operator:
\begin{equation}
(\mathcal{T}_\tau^* V)(s) := \sup_{a \in \mathcal{A}} \left[ r(s,a,\tau) + \gamma \mathbb{E}_{s' \sim P(\cdot|s,a)} V(s') \right].
\end{equation}

The utility of these operators in the multi-task setting stems from their stability across the task distribution $\mu$. Even as the reward function shifts with $\tau$, the underlying contraction property remains invariant: the standard MDP result that the Bellman operator is a $\gamma$-contraction in the sup-norm applies here as well, since the dependence on $\tau$ only affects the rewards and not the contraction induced by $\gamma \in (0,1)$.

This stability allows us to view the multi-task landscape not as a collection of disjoint problems, but as a single augmented MDP. By defining the state space as $\tilde{\mathcal{S}} = \mathcal{S} \times \mathcal{T}$, we can express the dynamics as
\begin{equation}
\tilde{P}((s',\tau) \mid (s,\tau), a) = P(s' \mid s,a)\,\delta(\tau),
\end{equation}
where the task remains fixed throughout the trajectory, effectively ``freezing'' the task context while the agent navigates the shared state space.

\subsection{Quadrotor Trajectory Following}

To ground the theoretical framework of multi-task reinforcement learning (MTRL) in a concrete physical system, we consider the control problem of quadrotor trajectory following in this paper. In this setting, the quadrotor must learn a generalized control policy capable of following diverse flight paths by treating each unique trajectory as a specific task $\tau$.

In the context of autonomous flight, the shared dynamics $P$ represent the task invariant physical properties of the quadrotor (e.g., mass, inertia, and aerodynamic coefficients). We define the components of the multi-task MDP as follows:

\paragraph{State Space $\mathcal{S}$} The task-invariant state $s \in \mathcal{S}$ captures the full kinematic and dynamic configuration of the vehicle:
\begin{equation}
s = [\mathbf{p}, \mathbf{v}, \mathbf{q}, \boldsymbol{\omega}]^\top \in \mathbb{R}^{13},
\end{equation}
where $\mathbf{p} \in \mathbb{R}^3$ is the position, $\mathbf{v} \in \mathbb{R}^3$ is the linear velocity, $\mathbf{q} \in \mathbb{H}$ is the unit quaternion representing orientation, and $\boldsymbol{\omega} \in \mathbb{R}^3$ is the angular velocity in the body frame.

\paragraph{Task Space $\mathcal{T}$} A task $\tau$ corresponds to a specific reference trajectory $\mathbf{p}_{\text{ref}}$. To facilitate efficient learning and ensure the task variable is informative, we encode $\tau$ using a hybrid representation:
\begin{enumerate}
    \item \textbf{Polynomial Coefficients:} The global structure of the trajectory is defined by coefficients of a polynomial trajectory (as explained later in Section \ref{sec: experiments}), providing a compact time-invariant global context.
    \item \textbf{Sliding Reference Window:} To simplify the learning task, we include a long-horizon temporal window of future waypoints $\{\mathbf{p}_{\text{ref},t+k}\}_{k=1}^N$. This provides the policy with immediate local geometry, such as upcoming curvature or heading changes. 
\end{enumerate}

\paragraph{Reward Function $r(s,a,\tau)$} The reward is task-dependent as it penalizes deviations from the specific reference $\tau$. This formally is captured by:
\begin{equation}
r(s, a, \tau) = -\Big( w_1 \|\mathbf{p} - \mathbf{p}_{\text{ref}}\|_2^2 + w_2 \|\mathbf{v} - \mathbf{v}_{\text{ref}}\|_2^2 + w_3 h(s) \Big),
\end{equation}
where the weights $w_i$ balance tracking precision against control effort, and $h(s)$ is a function designed to encourage stable flight behaviors and avoid aggressive maneuvers.

By utilizing the augmented state $\tilde{s} = (s,\tau)$, the quadrotor must navigate a higher-dimensional manifold where the ``frozen'' task variable $\tau$ acts as a constant parameter throughout the episode.

\section{Spectral Decomposition Multi-Task MDP Extension}
\label{sec: mt mdp sd}

We now extend the spectral decomposition of MDPs to the multi-task setting defined above. The goal is to represent task-conditioned value functions in a low-dimensional, linearly parameterized form while separating task-invariant dynamics from task-specific rewards.

\begin{definition}[Spectral Decomposition of MTRL MDPs]
Let $\mathcal{M}^{(\tau)}$ be a multi-task MDP as defined in Section~\ref{sec: preliminaries}. We say that $\mathcal{M}^{(\tau)}$ admits a \emph{spectral decomposition} if there exists a task-invariant feature map
\[
\phi: \mathcal{S} \times \mathcal{A} \to \mathbb{R}^d
\]
and task-dependent reward parameter mapping
\[
\theta: \mathcal{T} \to \mathbb{R}^d
\]
such that the reward function can be expressed as a linear combination of features:
\begin{equation}
r(s, a, \tau) = \langle \phi(s, a), \theta(\tau) \rangle, \quad \forall s \in \mathcal{S}, a \in \mathcal{A}, \tau \in \mathcal{T}.
\end{equation}
\end{definition}

\begin{lemma}[Spectral Decomposition of Task $Q$-Function]
For any task-conditioned policy $\pi$, the action-value function $Q^\pi(s, a; \tau)$ admits a linear parameterization:
\begin{equation}
Q^\pi(s, a; \tau) = \langle \phi(s, a), w^\pi(\tau) \rangle,
\label{eq:linearQ}
\end{equation}
where the weight vector $w^\pi(\tau)$ captures both the immediate task-specific reward and the discounted expected future returns under $\pi$.
\end{lemma}

\begin{proof}
Starting from the task-conditioned Bellman equation:
\begin{align*}
Q^\pi(s, a; \tau) 
&= r(s, a, \tau) + \gamma \, \mathbb{E}_{s' \sim P(\cdot \mid s, a)} [V^\pi(s'; \tau)] \\
&= \langle \phi(s, a), \theta(\tau) \rangle + \gamma \, \mathbb{E}_{s' \sim P(\cdot \mid s, a)} \big[ \mathbb{E}_{a' \sim \pi(\cdot \mid s', \tau)} [Q^\pi(s', a'; \tau)] \big].
\end{align*}

Assuming $Q^\pi(s', a'; \tau) = \langle \phi(s', a'), w^\pi(\tau) \rangle$, we have
\begin{align*}
Q^\pi(s, a; \tau) 
&= \langle \phi(s, a), \theta(\tau) \rangle + \gamma \, \mathbb{E}_{s', a'} [\langle \phi(s', a'), w^\pi(\tau) \rangle] \\
&= \langle \phi(s, a), \underbrace{\theta(\tau) + \gamma \, \mathbb{E}_{s', a'}[\phi(s', a')^\top w^\pi(\tau)]}_{w^\pi(\tau)} \rangle.
\end{align*}
Hence, $Q^\pi(s, a; \tau)$ admits the linear form~\eqref{eq:linearQ}. 
\end{proof}

Note that this result extends classical spectral or linear $Q$-function representations~\cite{zhang2022making} in a critical way: whereas prior work assumes that both the feature map $\phi(s,a)$ and the weight vector $w$ are \emph{fixed constants} independent of the task, our decomposition treats $w^\pi(\tau)$ as an explicit \emph{function of the task descriptor} $\tau$. Concretely, in single-task or task-agnostic settings, one writes $Q^\pi(s,a) = \langle \phi(s,a), w^\pi \rangle$ where $w^\pi \in \mathbb{R}^d$ is a single vector summarizing the value structure of the unique reward signal. In the multi-task setting, however, each task $\tau$ induces a distinct reward parameter $\theta(\tau)$ and, consequently, a distinct weight vector $w^\pi(\tau)$. The feature map $\phi(s,a)$ remains \emph{shared and task-invariant}, encoding the geometry of the state-action space, while all task-specific information is absorbed into $w^\pi(\tau)$. 
\begin{figure*}[t]  
    \centering
    \includegraphics[width=0.9\textwidth]{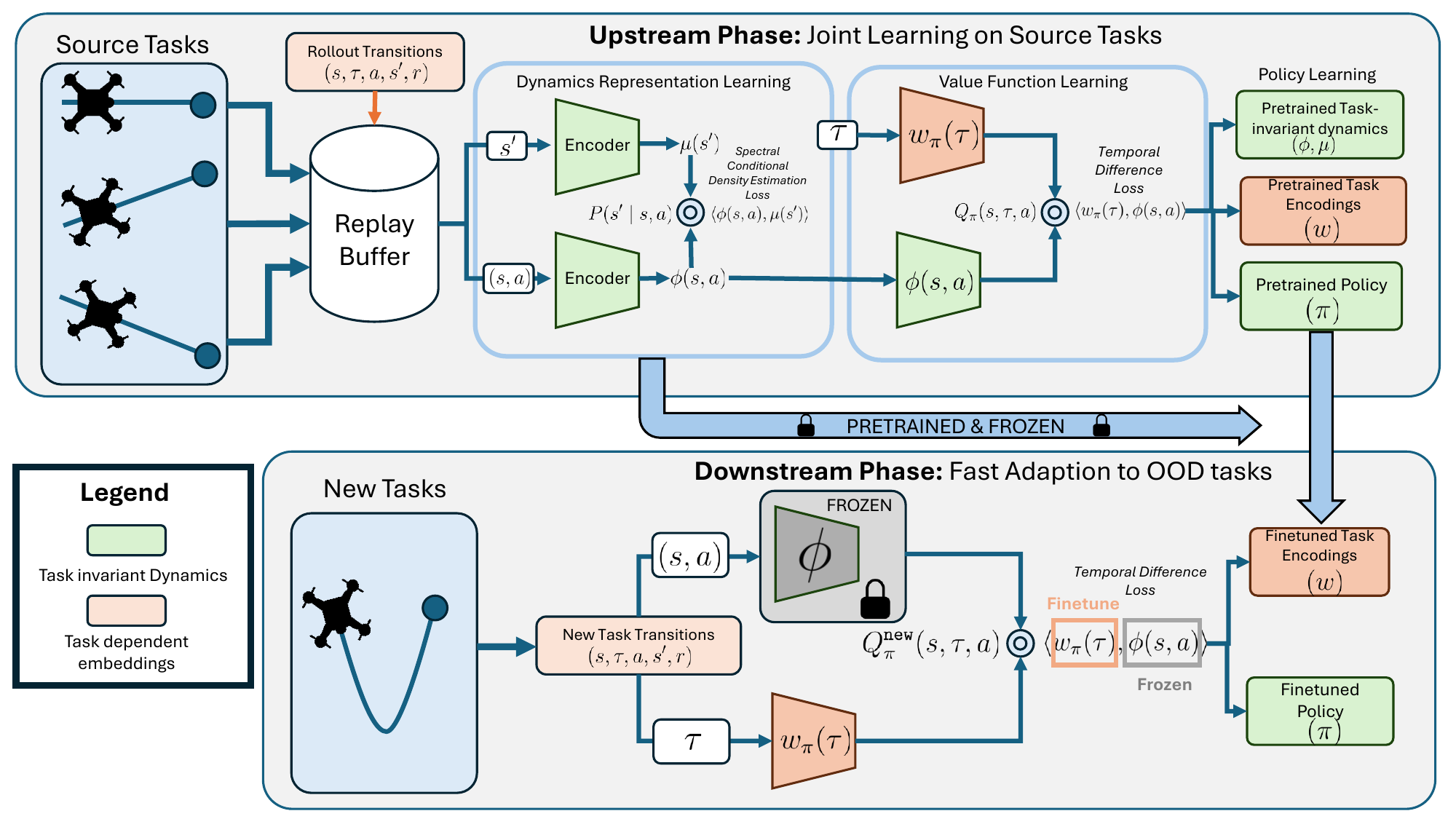}  
    \caption{Illustration of the two-phase learning paradigm. In the upstream phase (top), task-invariant dynamics representations and task-dependent embeddings are jointly learned from replayed transitions across source tasks. In the downstream phase (bottom), the learned dynamics are frozen while task-specific embeddings and the policy are fine-tuned for rapid adaptation to unseen tasks.}
    \label{fig: pipeline}
\end{figure*}
\textbf{Connection to the MTRL Objective:}
Given the linear representation of $Q^\pi(s, a; \tau)$, the task-specific value function is
\[
V^\pi(s; \tau) = \mathbb{E}_{a \sim \pi(\cdot \mid s, \tau)} [ Q^\pi(s, a; \tau) ] 
= \mathbb{E}_{a \sim \pi(\cdot \mid s, \tau)} [ \langle \phi(s, a), w^\pi(\tau) \rangle ].
\]

Consequently, the multi-task reinforcement learning (MTRL) objective can be expressed as
\begin{align}
J(\pi) &= \mathbb{E}_{\tau \sim \mu} \Big[ \mathbb{E}_{s_0 \sim \rho} \big[ V^\pi(s_0; \tau) \big] \Big] 
\\ &= \mathbb{E}_{\tau \sim \mu} \Big[ \mathbb{E}_{s_0 \sim \rho, a_0 \sim \pi(\cdot \mid s_0, \tau)} [ \langle \phi(s_0, a_0), w^\pi(\tau) \rangle ] \Big].
\label{eq:Jlinear}
\end{align}

Equation~\eqref{eq:Jlinear} explicitly shows that maximizing the expected return across tasks reduces to learning task-dependent weights $w^\pi(\tau)$ while leveraging the shared, task-invariant features $\phi(s, a)$.

\section{RepMT-SAC: Representation-based Multi-Task RL}
\label{sec:repmt}

In this section, we propose \textbf{RepMT-SAC}, a novel variant of the Soft Actor-Critic (SAC) algorithm designed for multi-task reinforcement learning. RepMT-SAC leverages the proposed MTRL parameterization to disentangle learning task-invariant environment dynamics from task-specific objectives, enabling efficient reuse of learned dynamics across tasks. As described in Figure \ref{fig: pipeline}, in the upstream phase, the framework jointly optimizes task-agnostic representations of the environment and task-dependent encodings of the Q-function, from which a maximum-entropy policy is derived. During downstream adaptation, the learned dynamics are frozen, and only the task encodings and policy are updated to rapidly specialize to new tasks. This two-phase approach allows RepMT-SAC to achieve strong zero-shot performance on in-distribution tasks and fast adaptation to out-of-distribution tasks, improving sample efficiency and generalization compared to standard multi-task RL baselines.

\subsection{Upstream Phase: Learning Task-Invariant Representations}
In the upstream phase, RepMT-SAC focuses on capturing reusable, task-invariant dynamics while simultaneously learning task-specific encodings that parameterize the Q-function. Concretely, we decompose the state-action value function for any task $\tau$ as specified in equation \ref{eq:Jlinear}, where $\phi(s, a)$ represents the task-agnostic dynamics, and $w_\pi(\tau)$ encodes task-dependent contributions. We follow the spectral conditional density estimation framework, which aims to learn task-invariant dynamics representations $(\phi, \mu)$ by approximating the conditional transition density $P(s' \mid s, a)$. Formally, the ideal objective is
{\small
\begin{align}
    \min_{\phi, \mu} \mathbb{E}_{(s,a) \sim \mathcal{D}, s' \sim P(\cdot \mid s,a)} \Big[ \big\| P(s^\prime \mid s,a) - \langle \phi(s,a), \mu(s') \rangle \big\|_2^2 \Big]
\end{align}
}
where $\mathcal{D}$ is the replay buffer of collected transitions. However, this is typically intractable, since $P(s' \mid s,a)$ is unknown.  

To address this, we adopt a sampling-based approximate loss based on the cross-entropy loss proposed in \cite{zhang2022making}:
{\footnotesize
\begin{align}
    - \mathbb{E}_{(s,a) \sim \mathcal{D}, s' \sim P(\cdot \mid s,a)} \Big[ 
        \log \frac{\exp\big(\langle \phi(s,a), \mu(s') \rangle\big)}{\sum_{s'' \in \mathcal{D}} \exp\big(\langle \phi(s,a), \mu(s'') \rangle\big)} 
    \Big]  
    \label{eq: feature loss}
\end{align}
}

Subsequently, the task encoding $w(\tau;\theta)$ is updated via a temporal-difference (TD) objective:
\begin{align}
    \min_{\theta} \; 
\mathbb{E}
\Big[ 
\big( r + \gamma \bar{Q}(s^\prime, \tau, \pi(s^\prime; \tau)) - w(\tau;\theta)^\top \phi(s, a) \big)^2 
\Big]
\label{eq: td loss}
\end{align}
where we sample $(s, \tau, a,r,s')$ from $\mathcal{D}$ and $\bar{Q}$ is the target Q-function commonly used in the target network trick \cite{mnih2015human}. The advantage of this parameterization is that, compared to traditional deep Q-learning \cite{mnih2013playing}, the policy evaluation optimization is a linear function, yielding much more stable learning. The task-invariant representation $\phi(s, a)$ is learned jointly with the TD loss, ensuring it captures the underlying dynamics shared across all tasks. 

Finally, a maximum-entropy policy $\pi(a \mid s_\circ, s_\tau)$ is derived from the learned Q-function following the Soft Actor-Critic (SAC) framework. Specifically, the policy is parameterized to maximize both expected return and entropy, encouraging exploration and robustness:
\begin{align}
    \pi(a \mid s_\circ, s_\tau) \propto \exp \Big( \frac{1}{\alpha} Q(s_\circ, s_\tau, a) \Big)
    \label{eq: policy}
\end{align}
where $\alpha$ is the temperature coefficient controlling the trade-off between reward maximization and entropy. 

In practice, we implement this using the reparameterization trick, which allows efficient gradient-based updates of the policy parameters. As summarized in algorithm \ref{alg:repmt upstream}, by combining the task-invariant representations $\phi(s, a)$ with task-specific encodings $w(\tau)$ in the Q-function, this approach produces a flexible, reusable policy that can be rapidly adapted during downstream task-specific finetuning.
\begin{algorithm}[H]
\caption{Upstream Phase of RepMT-SAC}
\begin{algorithmic}[1]
\State \textbf{Input:} MDPs for source tasks $\mathcal{M}^{\tau_1}, \dots, \mathcal{M}^{\tau_k}$
\State Initialize task-invariant representations $\phi, \mu$, task encodings $w_\theta$, critics $Q_1, Q_2$, policy $\pi$, target networks $\bar{Q}_1, \bar{Q}_2$
\For{episode $n = 1$ to $N$}
    \State Collect transitions $(s, \tau, a, r, s')$ and store in replay buffer $\mathcal{D}$
    \State Update task-invariant representations $(\phi, \mu)$ by minimizing \ref{eq: feature loss}
    \State Update task encoding $w_\theta$ via \ref{eq: td loss}
    \State Update policy $\pi$ according to \ref{eq: policy}
\EndFor
\State \textbf{Output:} Task-invariant representations $\phi, \mu$, task encodings $w_\theta$, policy $\pi$
\end{algorithmic}
\label{alg:repmt upstream}
\end{algorithm}
\subsection{Downstream Adaptation to Unseen Tasks}
Once the upstream phase has learned robust task-invariant dynamics $\phi(s, a)$ and the auxiliary representation $\mu(s')$, RepMT-SAC can efficiently adapt to new, unseen task distribution $\hat{\mu}$. In this downstream phase, we freeze the task-invariant representations and only optimize the task-specific encoding $w(\tau;\theta)$ and the policy $\pi(a \mid s, \tau)$. This significantly reduces the number of parameters to train and improves sample efficiency.

By decoupling task-invariant dynamics from task-specific contributions, this downstream adaptation procedure allows RepMT-SAC to achieve fast specialization with very few environment interactions, enabling strong zero-shot performance when combined with task embeddings inferred from limited data. Algorithm \ref{alg:repmt downstream} summarizes this process.

\begin{algorithm}[H]
\caption{Downstream Adaptation of RepMT-SAC}
\begin{algorithmic}[1]
\State \textbf{Input:} New task MDP $\mathcal{M}^{\tau_\text{new}}$, frozen representations $\phi, \mu$, and pretrained $w$ and $\pi$
\For{episode $n = 1$ to $N_\text{adapt}$}
\State Collect transitions $(s, \tau, a, r, s')$ in the new task
\State Update $w_{\tau_\text{new}}$ using TD loss with frozen $\phi$ and $\mu$ according to \ref{eq: td loss}
\State Update policy $\pi_{\tau_\text{new}}$ according to maximum-entropy objective
\EndFor
\State \textbf{Output:} Adapted task encoding $w_{\tau_\text{new}}$, policy $\pi_{\tau_\text{new}}$
\end{algorithmic}
\label{alg:repmt downstream}
\end{algorithm}

This two-stage decomposition—learning reusable dynamics upstream and fast task-specific adaptation downstream—forms the core of RepMT-SAC, enabling both multi-task generalization and rapid adaptation to previously unseen environments.
\section{Experiments}
\label{sec: experiments}

\begin{table*}[t]
\centering
\caption{Performance comparison of CTRL, SAC, and RepMT methods across Source, In-Distribution (ID), and Out-of-Distribution (OOD) tasks. Metrics include mean reward (R) and success rate (S), with OOD results shown both before and after fine-tuning. Bold values indicate the best performance for each metric.}
\label{tab:results_last_finetune_full}
\begin{tabular}{lcccccccc}
\toprule
Method 
& Source R 
& Source S 
& ID R 
& ID S 
& OOD R (Before) 
& OOD S (Before) 
& OOD R (After) 
& OOD S (After) \\
\midrule
CTRL 
& 377.9 $\pm$ 35.4 
& 0.0 $\pm$ 0.0 
& 432.1 $\pm$ 128.4 
& 0.0 $\pm$ 0.0 
& 300.8 $\pm$ 114.2 
& 0.0 $\pm$ 0.0 
& 481.1 $\pm$ 183.6 
& 0.0 $\pm$ 0.0 \\
SAC 
& 857.1 $\pm$ 277.9 
& 60.0 $\pm$ 54.8 
& 876.1 $\pm$ 294.4 
& 60.0 $\pm$ 54.8 
& 524.4 $\pm$ 180.4 
& 26.7 $\pm$ 27.9 
& 775.6 $\pm$ 155.6
& 40.0 $\pm$ 54.8 \\
RepMT 
& \textbf{1094.2 $\pm$ 17.9} 
& \textbf{100.0 $\pm$ 0.0} 
& \textbf{1100.7 $\pm$ 16.1} 
& \textbf{100.0 $\pm$ 0.0} 
& 577.8 $\pm$ 388.1 
& 53.3 $\pm$ 50.6 
& \textbf{1034.8 $\pm$ 23.3} 
& \textbf{100.0 $\pm$ 0.0} \\
\bottomrule
\end{tabular}
\end{table*}

We evaluate the proposed  multi-task reinforcement learning framework for quadrotor trajectory following using the formalism established in Section~\ref{sec: preliminaries} in IsaacSim. Each trajectory corresponds to a task $\tau \in \mathcal{T}$, and the quadrotor must learn a policy $\pi(s, \tau)$ that generalizes across both in-distribution and out-of-distribution trajectories.

\begin{figure}[ht]
    \centering
    \includegraphics[width=\linewidth]{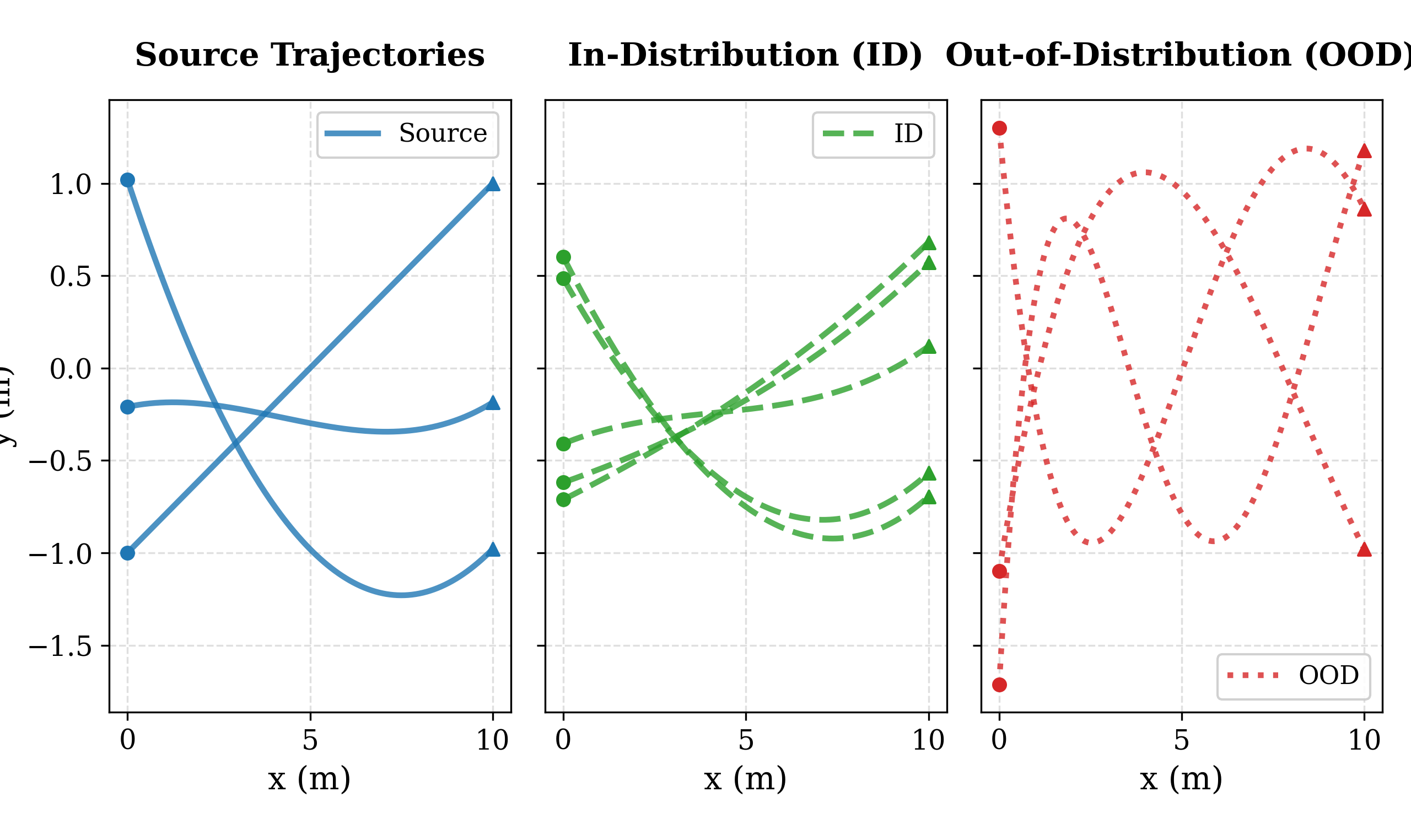}
    \caption{Visualization of the task decomposition used for evaluation. Source tasks form a basis set used during training, in-distribution tasks are convex combinations within this span, and out-of-distribution tasks extend beyond it via higher-order legendre polynomials.}
    \label{fig:trajectory}
\end{figure}

\subsection{Trajectory Representation and Task Parameterization}

Following Section~\ref{sec: preliminaries}, we encode each task $\tau$ as a parameterized polynomial trajectory. Formally, each trajectory is represented by a vector of coefficients:
\[
\tau = [c_0, c_1, \dots, c_d],
\]
with the continuous trajectory defined by
\[
f(t) = \sum_{i=0}^{d} c_i \psi_i(t), \quad \psi_i(t) \text{ Legendre polynomials \cite{ross2004legendre}.}
\]

This representation aligns directly with the task-dependent reward $r(s,a,\tau)$ and the augmented state $\tilde{s} = (s, \tau)$ described in Section~\ref{sec: preliminaries}. The coefficients $\{c_i\}$ serve as a compact embedding of the task variable for the policy $\pi(s, \tau)$, facilitating both learning and generalization.

\subsection{Task Hierarchy and Evaluation Distributions}

To assess generalization, we define three task sets consistent with our earlier notation and visualized in Figure \ref{fig:trajectory}:

\begin{enumerate}
    \item \textbf{Source Tasks} $\widehat{\mathcal{T}} \subset \mathcal{T}$: A set of fundamental polynomial basis functions (e.g., $\psi_0, \psi_1, \psi_2$) used for upstream training. The policy is trained on these tasks, corresponding to a uniform sampling distribution $\mu_{\text{train}}$ over $\widehat{\mathcal{T}}$.
    
    \item \textbf{In-Distribution Tasks} $\mathcal{T}^{\text{ID}} \subset \mathcal{T}$: Tasks formed as convex combinations of the source tasks,
    \[
    \tau^{\text{ID}} = \sum_{i \in \widehat{\mathcal{T}}} \alpha_i \psi_i, \quad \alpha_i \in [-1,1], \quad \sum_i \alpha_i = 1.
    \]
    These tasks remain in the span of the source basis, corresponding to the distribution $\mu_{\text{ID}}$, and allow us to measure generalization within the learned task manifold.
    
    \item \textbf{Out-of-Distribution Tasks} $\mathcal{T}^{\text{OOD}} \subset \mathcal{T}$: Higher-degree polynomials or tasks with coefficients outside the training range (e.g., $\psi_3, \psi_4, \psi_5$), defining $\mu_{\text{OOD}}$. Performance on these tasks evaluates the policy’s ability to transfer beyond the span of source tasks.
\end{enumerate}

In this framework, the polynomial coefficients naturally define the \textit{task-dependent component} of the augmented state, $\tau$, while the quadrotor’s odometry, orientation, velocity, and angular rates correspond to the \textit{task-invariant component}, $s$, as introduced in Section~\ref{sec: preliminaries}.

\subsection{Evaluation Protocol}

For each framework we run 5 trials across different seeds in order to evaluate the learned policies on:

\begin{itemize}
    \item \textbf{Source Task Performance:} Average return over $\widehat{\mathcal{T}}$ sampled from $\mu_{\text{train}}$, measuring learning performance.
    
    \item \textbf{In-Distribution Performance:} Average return over $\mathcal{T}^{\text{ID}}$ sampled from $\mu_{\text{ID}}$, measuring generalization within the span of training tasks.
    
    \item \textbf{Out-of-Distribution Performance:} Average return over $\mathcal{T}^{\text{OOD}}$ sampled from $\mu_{\text{OOD}}$, assessing transfer and adaptability to unseen trajectories.
    
    \item \textbf{Fine-Tuning:} For OOD tasks, we further apply downstream fine-tuning of the policy, optimizing
    \[
    J_{\text{OOD}}(\pi) = \mathbb{E}_{\tau \sim \mu_{\text{OOD}}} \left[ \mathbb{E}_{s_0 \sim \rho} V^\pi(s_0;\tau) \right],
    \]
    highlighting the utility of the learned task embedding $w(\tau)$ for rapid adaptation.
\end{itemize}

This experimental setup ties directly to the multi-task MDP formalism: the augmented state $\tilde{s} = (s, \tau)$, task-conditioned policy $\pi(s,\tau)$, and reward $r(s,a,\tau)$ follow exactly from the definitions and lemmas established in Section~\ref{sec: preliminaries}. The hierarchy of source, ID, and OOD tasks allows a clear quantification of generalization and transfer across the task distribution $\mu$.
\subsection{Baselines and Ablations}

We compare our approach against strong off-policy reinforcement learning baselines and ablations designed to isolate the contributions of our method. In particular, we use \textbf{Soft Actor-Critic (SAC)} \cite{DBLP:journals/corr/abs-1801-01290} as our primary baseline, as it represents the state-of-the-art for continuous control and off-policy controller learning.

To better understand the impact of our proposed representation structure, we also consider an ablated variant of our method, referred to as \textbf{CTRL-SAC} \cite{zhang2022making}. This variant incorporates a standard representation learning objective within the SAC framework but does not explicitly enforce the task-structured decomposition introduced in our approach.

Together, these baselines allow us to evaluate both overall performance gains and the specific benefits of our representation-based multi-task formulation.
\subsection{Performance on Source Tasks}
\begin{figure}[ht]
    \centering
    \includegraphics[width=\linewidth]{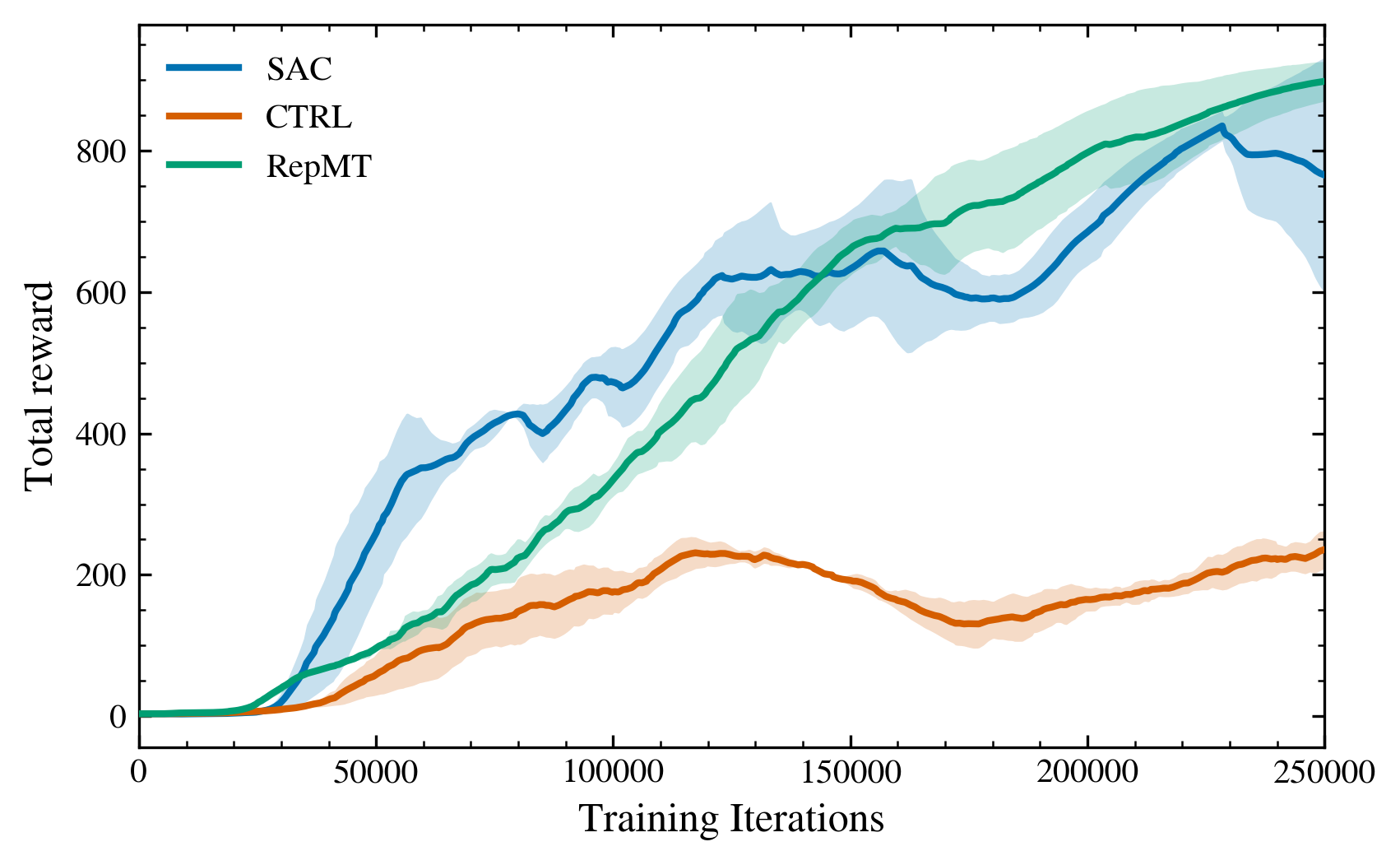}
    \caption{Comparison of learning curves across methods during upstream training. RepMT demonstrates stable, monotonic improvement with low variance, while baseline methods exhibit higher variance and less consistent convergence.}
    \label{fig:pretrain-rewards}
\end{figure}

The performance on source tasks $\widehat{\mathcal{T}}$ highlights the superior stability and asymptotic performance of the \texttt{RepMT} framework. As illustrated in Fig. \ref{fig:pretrain-rewards}, \texttt{RepMT} exhibits a near-monotonic increase in total reward with significantly lower variance compared to the \texttt{SAC} baseline, which suffers from pronounced oscillations and performance degradation after approximately $160,000$ iterations. Quantitatively, Table~\ref{tab:results_last_finetune_full} shows that \texttt{RepMT} achieves a perfect success rate of $100.0 \pm 0.0\%$ and a total reward of $1094.2 \pm 17.9$, significantly outperforming both \texttt{SAC} ($857.1 \pm 277.9$) and the \texttt{CTRL} ablation ($377.9 \pm 35.4$). This performance gain in fewer pretraining iterations is largely attributable to our task-structured value function parameterization $V(s; \tau)$, which utilizes the Legendre polynomial coefficients to provide a more consistent and informative gradient signal than standard off-policy methods.

\subsection{Performance on In-Domain Tasks}
To evaluate the policy's ability to generalize within the learned task manifold, we tested the frameworks on in-distribution (ID) tasks $\mathcal{T}^{\text{ID}}$ consisting of convex combinations of the source basis functions. As shown in Table~\ref{tab:results_last_finetune_full}, \texttt{RepMT} demonstrates exceptional zero-shot generalization, maintaining a $100.0 \pm 0.0\%$ success rate and a mean reward of $1100.7 \pm 16.1$. 

This performance represents a significant improvement over the baselines. While \texttt{SAC} achieves a mean reward of $876.1$, it suffers from high variance ($\pm 294.4$) and a success rate of only $60.0 \pm 54.8\%$, suggesting that standard off-policy reinforcement learning fails to consistently cover the interpolated task space. More notably, the \texttt{CTRL} ablation—which incorporates representation learning without our task-structured decomposition—fails to solve the ID tasks entirely, plateauing at a $0.0\%$ success rate. 

The contrast between \texttt{RepMT} and \texttt{CTRL} underscores that the explicit partitioning of the augmented state $\tilde{s} = (s, \tau)$ is essential for effective multi-task learning. By encoding tasks via a compact coefficient vector, the \texttt{RepMT} policy does not merely overfit to individual source trajectories but instead learns a continuous control manifold. This allows the quadrotor to interpolate seamlessly between fundamental maneuvers, confirming that our structured embedding captures the shared dynamics necessary for robust execution across the entire task distribution $\mu_{\text{ID}}$.

\begin{figure}[t]
\centering

\begin{subfigure}{\columnwidth}
    \centering
    \includegraphics[width=\linewidth]{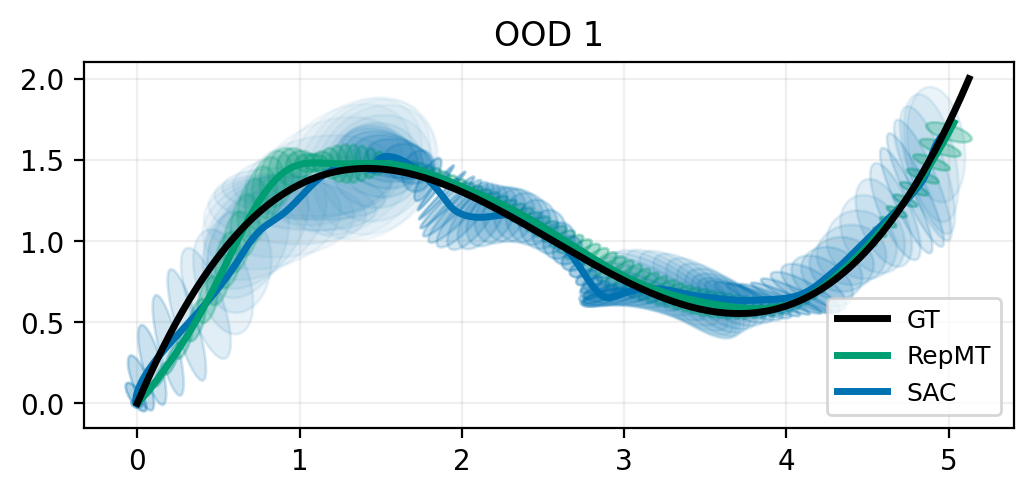}
    \caption{Performance on an OOD trajectory defined by a degree-4 Legendre polynomial.}
\end{subfigure}

\vspace{0.5em}

\begin{subfigure}{\columnwidth}
    \centering
    \includegraphics[width=\linewidth]{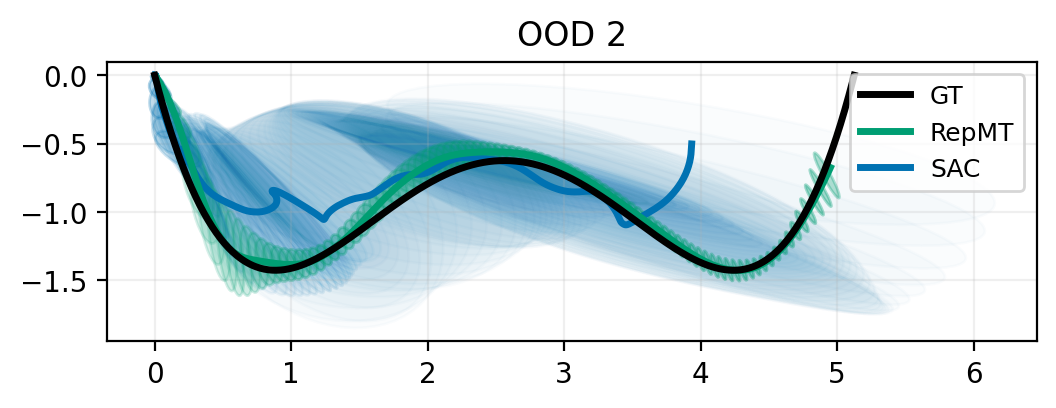}
    \caption{Performance on an OOD trajectory defined by a degree-5 Legendre polynomial.}
\end{subfigure}

\vspace{0.5em}

\begin{subfigure}{\columnwidth}
    \centering
    \includegraphics[width=\linewidth]{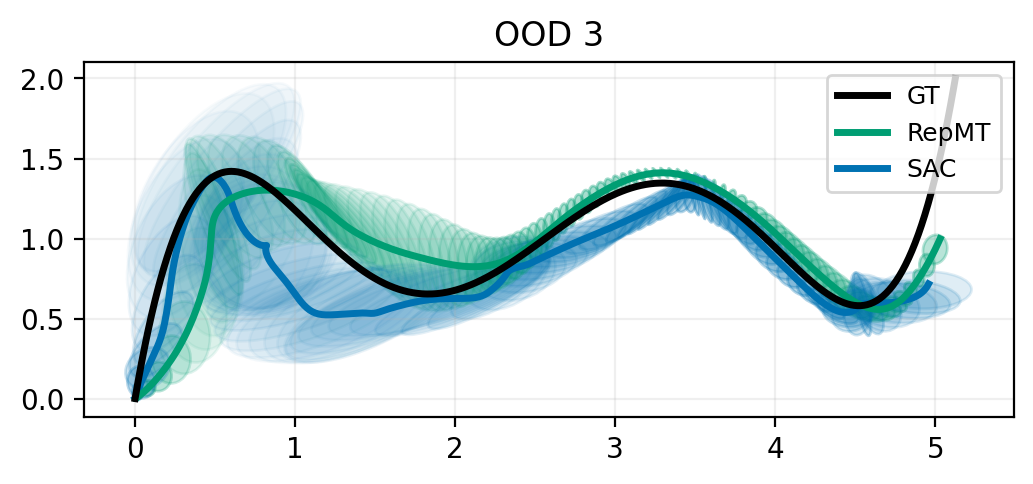}
    \caption{Performance on an OOD trajectory defined by a degree-6 Legendre polynomial.}
\end{subfigure}

\caption{Comparison of quadrotor trajectories under different methods after finetuning on OOD tasks. RepMT closely follows the reference trajectory, while SAC show deviations and instability, even after downstream adaptation.}
\label{fig:OOD_adaptation}
\end{figure}

\subsection{Downstream Adaptation to Out-of-Distribution Tasks}

Finally, we evaluate the adaptability of the learned representations through downstream fine-tuning on OOD tasks $\mathcal{T}^{\text{OOD}}$. Following the pretraining phase, each model is fine-tuned for only $25,000$ iterations—merely $10\%$ of the initial training budget. As shown in Table~\ref{tab:results_last_finetune_full}, \texttt{RepMT} exhibits substantially faster convergence and superior final performance, achieving a perfect $100.0 \pm 0.0\%$ success rate and a total reward of $1034.8 \pm 23.3$. 

In contrast, the baselines struggle to adapt within this limited timeframe. While \texttt{SAC} shows marginal improvement, rising from $26.7\%$ to $40.0\%$ success, it remains unable to reliably follow the higher-degree polynomial trajectories. The \texttt{CTRL} baseline fails to achieve any successful completions ($0.0\%$), despite a slight increase in total reward. This performance gap is qualitatively visualized in Fig. \ref{fig:OOD_adaptation}, where the \texttt{RepMT} trajectories (green) closely track the ground truth (black) across various OOD profiles, while \texttt{SAC} (blue) frequently deviates or fails to reach the target.

The efficiency of this adaptation highlights the utility of our learned task embedding $w(\tau)$. Because the \texttt{RepMT} value function is parameterized to recognize the underlying structure of the trajectory space, fine-tuning becomes a matter of rapid weight adjustment within a well-structured manifold, rather than a tabula rasa search for a new control policy. This demonstrates that our framework provides a robust foundation for transferring zero-shot capabilities to novel, complex maneuvers with minimal additional data.

\section{Conclusion}
\label{sec:conclusion}
In this work, we introduced \texttt{RepMT}, a multi-task reinforcement learning framework for quadrotor trajectory following that leverages a structured decomposition of task-invariant and task-dependent state components. By parameterizing the value function using Legendre polynomial coefficients, our approach provides a stable and informative gradient signal, leading to superior asymptotic performance and generalization. \texttt{RepMT} achieves a $100\%$ success rate on in-distribution tasks and enables rapid adaptation to complex, out-of-distribution trajectories with only $10\%$ of the original pretraining budget.

A limitation of the current work is the hand-crafted nature of $\mu_{\text{train}}$. Future work will focus on leveraging the learned task latent embeddings $w(\tau)$ to guide structured exploration within the task space and identify regions of the manifold that maximize policy robustness. Furthermore, we hypothesize that the stability and structured nature of the \texttt{RepMT} representations will facilitate more reliable Sim2Real transfer building on the work proposed in \cite{ma2024skill}. We plan to validate this by deploying the framework on physical quadrotor platforms and evaluating its resilience to real-world dynamics and environmental disturbances.


\bibliographystyle{IEEEtran}
\bibliography{references}

\end{document}